\def\year{2017}\relax
\documentclass[letterpaper]{article}
\usepackage{aaai17}
\usepackage{times}
\usepackage{helvet}
\usepackage{courier}
\usepackage{todonotes}
\usepackage{amsmath}
\usepackage{algorithm}
\usepackage{algpseudocode}
\usepackage{amsfonts}
\usepackage{amsthm}
\usepackage{bm}
\usepackage{comment}
\usepackage{graphicx}
\usepackage{amsmath}
\usepackage{siunitx}
\usepackage{multirow}
\usepackage{booktabs}
\usepackage[hyphens]{url}

\DeclareMathOperator{\E}{\mathbb{E}}
\newcommand{\alg}[1]{\texttt{#1}}
\newcommand{\vx}{\mathbf{x}}
\newcommand{\vy}{\mathbf{y}}
\newcommand{\eg}{e.g.}
\newcommand{\ie}{i.e.}
\newcommand{\Do}[1]{\todo[inline]{\textbf{TODO:}#1}}

\newcommand{\ouralg}{\alg{\textsc{EmbeddedHunter}}}
\newcommand{\optcell}{$\mathcal{Y}_{h,i^*_p}$}
\newcommand{\cell}{$\mathcal{Y}_{h,i}$}
\newcommand{\lcb}[2]{f^*_{#1,#2} - \tau(#1)- \delta(#1)} 
\newcommand\myeq{\mathrel{\stackrel{\makebox[0pt]{\mbox{\normalfont\tiny def}}}{=}}}

\newtheorem{lemma}{Lemma}
\newtheorem{assumption}{Assumption}
\newtheorem{theorem}{Theorem}

\newtheorem{definition}{Definition}
 
\begin{document}
%
\title{
	Embedded Bandits for Large-Scale Black-Box Optimization
	}

\author{Abdullah Al-Dujaili$^1$ \and S. Suresh$^{1,2}$ 
	\\ $^1$ School of Computer Science and Engineering, NTU,  Singapore \\
 $^2$ ST Engineering - NTU Corporate Lab, Singapore\\
 \tt aldujail001@e.ntu.edu.sg, ssundaram@ntu.edu.sg
}
\maketitle
\begin{abstract}
	Random embedding has been applied with empirical success to large-scale black-box optimization problems with low effective dimensions. This paper proposes the \ouralg~algorithm, which incorporates the technique in a hierarchical stochastic bandit setting, following the \emph{optimism in the face of uncertainty} principle and breaking away from the \emph{multiple-run} framework in which random embedding has been conventionally applied similar to stochastic black-box optimization solvers. 
	Our proposition is motivated by the bounded mean variation in the objective value for a low-dimensional point projected randomly into the decision space of Lipschitz-continuous problems. In essence, the \ouralg~algorithm expands optimistically a partitioning tree over a low-dimensional---equal to the effective dimension of the problem---search space based on a bounded number of random embeddings of sampled points from the low-dimensional space.
	In contrast to the probabilistic theoretical guarantees of multiple-run random-embedding algorithms, the finite-time analysis of the proposed algorithm~presents a theoretical upper bound on the regret as a function of the algorithm's number of iterations.  Furthermore, numerical experiments were conducted to validate its performance. The results show a clear performance gain over recently proposed random embedding methods for large-scale problems, provided the intrinsic dimensionality
	is low.
	
\end{abstract}

\section{Introduction}
\label{sec:intro}

\paragraph{Problem.} This paper is concerned with the \emph{large-scale black-box} optimization problem given a {finite number of function evaluations}. Mathematically, the problem has the form:
\begin{equation}
\begin{aligned}
& {\text{minimize}}
& & f(\vx) \\
& \text{subject to}
& & \vx \in \mathcal{X}\;,
\end{aligned}
\label{eq:problem_def}
\end{equation}
where  $f:\mathcal{X}\subseteq \mathbb{R}^n\to \mathbb{R}$ and $n\gg10^2$. Without loss of generality, it is assumed that $\mathcal{X}=[-1,1]^n$, and there exists at least one global optimizer $\vx^*$ whose  objective value is denoted by $f^*$, \ie, $\min_{\vx\in \mathcal{X}}f(\vx)=f(\vx^*)=f^*$. 
  Solving the optimization problem~\eqref{eq:problem_def} is notoriously difficult as the sole source of information about its objective function $f$
is available through a \emph{black-box} or an \emph{oracle}, which one can query for the value of $f$ at a specific solution (point ) $\vx \in \mathcal{X}$. High-order information~(\eg, derivatives) are unavailable symbolically nor numerically or are tedious to compute compared to zero-order information---\ie, point-wise function evaluations. Thus, the task is to find the (or one) optimal solution $\vx^* \in \mathcal{X}$ to~\eqref{eq:problem_def} or a good approximation using a finite number~$v$ of function evaluations, which is commonly referred to as the evaluation budget. The quality of the returned solution~$\vx(v)\in \mathcal{X}$ after $v$ function evaluations, denoted by $f^*_v$, is  assessed by the \emph{regret},
\begin{equation}
r(v) = f^*_v - f^*\;.\label{eq:simple-regret}
\end{equation}
Besides the aforementioned challenging nature of black-box problems, the high dimensionality $n$ of the decision space $\mathcal{X}$ poses another challenge towards finding the global optimum. Despite their witnessed success, the effectiveness of most black-box optimization algorithms is restricted to moderate dimensions (typically, $n<100$) and they do not scale well to high-dimensional (say, $n\gg 10^2$) problems. As the dimensionality increases, the number of evaluations (sampled points) required to cover $\mathcal{X}$ increases \emph{exponentially}. 

Despite the curse of dimensionality, it has been noted that for artificial intelligence (AI) applications, most dimensions of certain classes of the associated optimization problems do not affect the objective function significantly. In other words, such problems have \emph{low effective dimensionality}, e.g., hyper-parameter optimization for neural and deep belief networks~\cite{bergstra2012random}.

\vspace{0.5em}
\noindent\textbf{Related Work.} The literature on black-box optimization is huge and we only highlight here works that are closely related to the paper's contribution. The bulk of algorithmic work on large-scale black-box optimization has been following one of two approaches: \emph{decomposition} and \emph{embedding}. 

\textit{Decomposition} algorithms break the problem into several subproblems, and solutions for the original problem are recognized in a coordinated
	manner. In \cite{kandasamy2015high}, Bayesian optimization was scaled to high-dimensional problems whose objectives have an additive structure. \ie, the function $f$ is the sum of several sub-functions with smaller dimensions, such that no two sub-functions share one or more variables. On the other hand, \citeauthor{friesen2015recursive}~(\citeyear{friesen2015recursive}) proposed to decompose the function into \emph{approximately locally independent} sub-functions and optimize them separately.~\citeauthor{chen2010large}~(\citeyear{chen2010large}) addressed \emph{interdependent} sub-functions and proposed to consider all entries of the decision vector $\vx$ independent  and discover their relations gradually. 
	In general, decomposition methods employ axis-aligned decomposability, which may limit their applicability.

\textit{Embedding} algorithms exploit the assumption/empirical observation of \emph{low effective dimensionality}. \citeauthor{ICML2012Chen_714}~(\citeyear{ICML2012Chen_714}) presented a variable selection method to discover the effective axis-aligned subspace, while \citeauthor{Djolonga-high-bandits}~(\citeyear{Djolonga-high-bandits}) sought to learn the effective subspace using a low-rank matrix recovery technique. In \cite{carpentier2012bandit}, compressed sensing was applied to deal with linear-bandit problems with a high degree of sparsity. Recent works---motivated by the empirical success of random search in leveraging low effective dimensionality without knowing which variables are important~\cite{bergstra2012random}---presented \emph{random embedding} techniques based on the random matrix theory~\cite{wang2013bayesian,kaban} and provided probabilistic theoretical guarantees. In~\cite{qian2016scaling}, the Simultaneous Optimistic Optimization~(\alg{SOO}) algorithm~\cite{munos-soo} was scaled via random embedding. Problems, whose all dimensions are effective but many of them have a small bounded effect, were addressed in~\cite{qianderivative} where the random embedding technique was incorporated in a sequential framework. In general, random embedding methods employ multiple runs to substantiate the probabilistic theoretical performance. 

\vspace{0.5em}
\noindent\textbf{Our Contributions.} This paper aims to tackle large-scale black-box optimization~\eqref{eq:problem_def} based on the random embedding technique. Previous propositions put the technique in a framework of multiple runs---be it parallel~\cite{qian2016scaling} or sequential~\cite{qianderivative}---to maximize the performance guarantee. In this paper, we seek to break away from the \emph{multiple-run} framework and follow the \emph{optimism in the face of uncertainty} principle, or so-called \emph{optimistic optimization}. To this end, we incorporate the random embedding technique in a stochastic hierarchical bandit setting and present \ouralg: an algorithmic instance of the sought approach. Similar to other optimistic methods, \ouralg~iteratively expands a partitioning tree over a low-dimensional space $\mathcal{Y}$ based on  randomly projecting sampled points to the original high-dimensional space $\mathcal{X}$ once or more times. This approach is motivated by the proof that the mean variation in the objective function $f$ value for a point $\vy \in \mathcal{Y}$ projected randomly to $f$'s decision space $\mathcal{X}$ is bounded for objective functions that are Lipschitz-continuous. \ouralg's regret~\eqref{eq:simple-regret} is upper bounded in terms of the number of iterations required to expand near-optimal nodes in the (effective) low-dimensional space based on the Lipschitz continuity assumption and that random embedding can preserve local distance.


The rest of the paper is organized as follows. First, a formal motivation is presented, followed by an introduction to \ouralg. Then, the algorithm's finite-time performance is studied and complemented by an empirical validation. Towards the end, the paper is concluded.



\section{Optimistic Optimization Meets Random Embeddings}
\label{sec:motivation}

Optimistic methods, \ie, methods that implement the \emph{optimism in the face of uncertainty} principle have proved to be viable for black-box optimization. Such a principle finds its foundations in the machine learning field addressing the exploration-vs.-exploitation dilemma, known as the multi-armed bandit problem. 
 Within the
context of function optimization, optimistic approaches formulate the complex
problem of optimization~\eqref{eq:problem_def} over the space~$\mathcal{X}$ as a hierarchy of simple
bandit problems~\cite{kocsis2006bandit} in the form of space-partitioning tree search. At step $t$, the algorithm
optimistically expands a leaf node (partitions the corresponding subspace) that
may contain the global optimum. Previous empirical studies have shown that optimistic methods---e.g.,~\alg{SOO}~\cite{munos-soo} and \alg{NMSO}~\cite{AlDujaili2016nmso}---are not suitable for problems with high dimensionality.

Random embedding has emerged as a practical tool for large-scale optimization with an experimental success and probabilistic theoretical guarantees. It assumes the problem~\eqref{eq:problem_def} has an implicit low effective dimension $d$ much lower than the explicit (original) dimension $n$. In essence, for an optimizer $\vx^*\in \mathcal{X}=[-1,1]^n$ and a random matrix $A\in\mathbb{R}^{n\times d}$ whose entries are sampled independently from a normal distribution, there exists a point $\vy^*\in \mathcal{Y}=[-d/\eta, d/\eta]^d$ such that its Euclidean random projection to $\mathcal{X}$, $\mathcal{P}_{\mathcal{X}}(A\vy^*)$, is $\vx^*$ with a probability at least $1-\eta$ where $\eta\in(0,1)$. That is to say, 
$f(\mathcal{P}_{\mathcal{X}}(A\vy^*))=f(\vx^*)=f^*$. The Euclidean random projection of the $i$th coordinate $[\vy]_i$ to $[\mathcal{X}]_i$ is defined as follows.
\begin{equation}
[\mathcal{P}_{\mathcal{X}}(A\vy)]_i= 
\begin{cases}
1,& \text{if } [\vy]_i\geq 1;\\
-1,              & \text{if } [\vy]_i\leq -1;\\
[A\vy]_i		& \text{otherwise}.
\end{cases}
\label{eq:euclidean-projection}
\end{equation}
The reader can refer to \cite{wang2013bayesian,qian2016scaling} for more details and a formal treatment of the above.

It was shown that is possible to scale up optimistic methods via random embedding~\cite{qian2016scaling,qianderivative}. Nevertheless, one can observe that it has been applied in a \emph{multiple-run} framework, where (multiple) $M$ random embeddings  are applied on the \emph{same} low-dimensional search space $\mathcal{Y}$ in parallel or sequentially to increase the success rate $1-\eta^M$, ignoring the relationship among the function $f$ values at the multiple projections in $\mathcal{X}$ of a single point $\vy\in \mathcal{Y}$. In Theorem~\ref{lem:gy_def}, we show that a relationship can be established for Lipschitz functions. Prior to that, let us introduce some notation and state the Lipschitz condition formally.

\vspace{0.5em}
\noindent\textbf{Notation.} Let $\mathcal{N}$ denote the Gaussian distribution  with zero mean and $1/n$ variance, and $\{A_p\}_p\subseteq \mathbb{R}^{n\times d}$, with $d\ll n$, be a sequence of realization matrices of the random matrix $\mathbf{A}$ whose entries are sampled independently from $\mathcal{N}$. Furthermore, let $g_P(\vy)$ be a random (stochastic) function such that $g_P(\vy)\myeq f(\mathcal{P}_{\mathcal{X}}(\mathbf{A}\vy))$ and $g_p(\vy)=f(\mathcal{P}_\mathcal{X}(A_p\vy))$ is a realization (deterministic) function, where $\vy\in \mathcal{Y}\subseteq \mathbb{R}^d$. On the other hand, let us define $\ell(\vx_1,\vx_2)$ as $L\cdot ||\vx_1 - \vx_2 ||$, where $||\cdot||$ denotes the $L_2$-norm. The expectation of a random variable $X$ is denoted by $\mathbb{E}[X]$. 
\begin{assumption} $f$ is Lipschitz-continuous, \ie, $\forall \vx_1, \vx_2 \in \mathcal{X}$,
	\begin{equation}
	|f(\vx_1)-f(\vx_2)|\leq L \cdot ||\vx_1 - \vx_2||\;,\label{eq:lipschitz}
	\end{equation}
	where $L>0$ is the Lipschitz constant.\label{assmptn:lipschitz}
\end{assumption}
\begin{theorem}[Mean absolute difference for $g_{P}(\vy)$] $\forall y \in \mathcal{Y}\subseteq \mathbb{R}^d$, we have
	$E[|g_p(\vy) - g_q(\vy)|]\leq \sqrt{8}\cdot L \cdot ||\vy||\;.$\label{lem:gy_def}
\end{theorem}
\begin{proof} From Assumption~\ref{assmptn:lipschitz}, we have
	\begin{align}
	\E[|g_p(\vy) - g_q(\vy)|] = &\E[|f(\mathcal{P}_{\mathcal{X}}(A_p\vy)) - f(\mathcal{P}_{\mathcal{X}}(A_q\vy))|]\nonumber\\
	\leq & L \cdot \E[|| \mathcal{P}_{\mathcal{X}}(A_p\vy) - \mathcal{P}_{\mathcal{X}}(A_q\vy)   ||]\;. \nonumber \\
	\intertext{From the definition of the Euclidean projection~\eqref{eq:euclidean-projection}:}
	\E[|g_p(\vy) - g_q(\vy)|]\leq & L \cdot \E[|| A_p\vy- A_q\vy  ||] \nonumber \\
	\intertext{Thus, from Cauchy's inequality, we have}
	\E[|g_p(\vy) - g_q(\vy)|] \leq & L \cdot ||\vy|| \cdot \E[||A_p- A_q||] \label{eq:norm2-rand} \nonumber\\
	\leq & L \cdot ||\vy|| \cdot  \sqrt{\frac{8}{n}}\cdot \sqrt{\max(n,d)}\\
	\leq & \sqrt{8}\cdot L \cdot  ||\vy||\;, \nonumber
	\end{align}
	where~\eqref{eq:norm2-rand} is derived from \cite{hansen19882}.
\end{proof}
Theorem~\ref{lem:gy_def} says  that the variation in the function $f$ values at points in $\mathcal{X}$ projected randomly from the same low-dimensional point $\vy \in \mathcal{Y}$ is bounded on the order of the point's norm $||\vy||$. Indeed, the $d$-dimensional zero vector (center of $\mathcal{Y}$) will always give the same function value (zero variation), regardless of the random matrix used. As a result, one is motivated to project a point $\vy$ multiple times proportional to its norm in search for the optimal solution $\vx^*$. Next, we provide \ouralg: a novel scalable optimistic algorithm that exploits the above result.

\section{\ouralg}
\label{sec:body1}

\ouralg~is a space-partitioning tree-search algorithm that constructs iteratively finer and finer partitions of the (effective) low-dimensional space $\mathcal{Y}$ in a hierarchical fashion looking for the global optimum. The hierarchical partitioning can be represented by a $K$-ary tree $\mathcal{T}$, where nodes of the same depth $h$ correspond to a partition of $K^h$ subspaces/cells. \ie, the $i$th node at depth $h$, denoted by $(h,i)$, corresponds to the subspace/cell $\mathcal{Y}_{h,i}$ such that $\mathcal{Y}=\cup_{0\leq i < K^h}\mathcal{Y}_{h,i}$. To each node $(h,i)$, a base point $\vy_{h,i}$ (center of $\mathcal{Y}_{h,i}$) is assigned at which $f$ is evaluated once or more times.  That is to say, for every new evaluation of the node $(h,i)$, $\vy_{h,i}$ is randomly projected to $f$'s decision space $\mathcal{X}$ via a random matrix \eqref{eq:euclidean-projection} and gets evaluated. To expand/split a node, the corresponding cell is partitioned into $K$ subscells along one of $\mathcal{Y}$'s coordinates, one coordinate at a depth in a sequential manner. Moreover, let the set of leaf nodes of $\mathcal{T}$ be denoted as $\mathcal{L}$. Furthermore, one can denote the algorithm's tree $\mathcal{T}$ at step $t$ by $\mathcal{T}_{t}$. Towards the detailed aspects of the algorithm, some assumptions are made about the hierarchical partitioning in line with Assumption~\ref{assmptn:lipschitz} and Theorem~\ref{lem:gy_def}, relating variations in function $f$ values using the same and different projection matrices, respectively.

\vspace{0.5em}
\noindent\textbf{Function values within the same projection.} Based on the 
Johnson-Lindenstrauss Lemma~\cite{Achlioptas2003671,vempala2005random}; for a set of $m$ points $\{\vy^i\}_{0\leq i \leq m}\subset \mathcal{Y}$ and their projections $\{\vx^i\}_{0\leq i \leq m}\subset \mathcal{X}$ via the same matrix, we have $\ell(\vx^i,\vx^j) \leq (1+\epsilon)^{1/2} \cdot \ell(\vy^i,\vy^j)$, where $\epsilon \in (0,1/2]$ and $n > 9 \ln m / (\epsilon^2-\epsilon^3)$---see \cite[Lemma 3]{qian2016scaling}.

In other words, the random embedding can probably preserve local distance. Thus, from Assumption~\ref{assmptn:lipschitz}, the difference between the function $f$ values of two points in the low-dimensional space $\mathcal{Y}$---using the same projection matrix---is on the order of their distance in $\mathcal{Y}$. The next assumption exploits the above observation with respect to the optimal cell (node), which is defined as follows.

\begin{definition}[Optimal cell] A cell \cell~at depth $h\geq 0$ is optimal if there exists a random matrix $A_p$ whose entries are sampled independently from $\mathcal{N}$ such that $\min_{\vy\in \mathcal{Y}_{h,i}}g_p(\vy)=f(\vx^*)$, where $\vx^*$ is a global optimizer of $f$. We denote such a cell by \optcell and its node by $(h,i^*_p)$.
\end{definition}
\begin{assumption}[Bounded intra-variation] There exists a decreasing sequence $\delta$ in $h\geq 0$ such that for one (or more) optimal cell(s) \optcell~at depth $h$, we have
	$$0\leq \sup_{q,\vy \in \mathcal{Y}_{h,i^*_p}}|g_q(\vy_{h,i})-g_q(\vy)|\leq \delta(h)\;.$$\label{assmptn:deltah}
\end{assumption}
As the hierachical partitioning is performed coordinate-wise in a sequential manner, let us link the fact that the cell's shapes are not skewed in some dimensions with  Assumption~\ref{assmptn:deltah} through the next assumption.

\begin{assumption}[Well-shaped cells]
	\label{assmptn:rounded}	$\exists$ $m>0$ such that
	$\forall (h,i) \in \mathcal{T}$,~$\mathcal{Y}_{h,i}$ contains an $\ell$-ball of radius $m\delta(h)$ centered in $\vy_{h,i}$. 
\end{assumption}
\vspace{0.5em}
\noindent\textbf{Function values among different projections.} Now, we state another assumption about the optimal cell \optcell~in line with Theorem~\ref{lem:gy_def}.
\begin{assumption}[Bounded inter-variation]
	There exists two non-decreasing sequences $\lambda$ and $\tau$ in $\vy$ and $h$, respectively, such that for any depth $h\geq 0$, for any optimal cell \optcell,  $$0\leq \sup_{s,t}|g_s(\vy_{h,i^*_p})-g_t(\vy_{h,i^*_p})|\leq \lambda(\vy_{h,i^*_p})\;,$$
	and  $\sup_{i^*_p} \lambda(\vy_{h,i^*_p}) \leq \tau(h)\;.$\label{assmptn:lambday}
\end{assumption}
Note that $\lambda$ being bounded by $\tau$ in $h$ is due to the nature of the hierarchical partitioning: the maximum norm of base points at depth $h$ is smaller than or equal to those at depth $h+1$. \eg, at depth $h=0$, there is a single node $(0,0)$ whose base point is the $d$-dimensional zero vector centered in $\mathcal{Y}$. As the tree $\mathcal{T}$ goes deeper, more base points farther away from the center---and hence greater norms---are sampled.

Combining Assumptions~\ref{assmptn:deltah} and \ref{assmptn:lambday} implies that the values of function $f$, which the optimal node's base point $\vy_{h,i^*_p}$ can have, are within $\tau(h)+\delta(h)$ from the global optimum $f^*$. One can therefore establish a lower confidence bound (commonly referred to as the $b$-value) on the $f$ values within a cell. Let $f^*_{h,i}$ be the best function $f$ value achieved among $\vy_{h,i}$ evaluations. Then, the $b$-value for $(h,i)$ can be written as $b_{h,i}\myeq \lcb{h}{i}$ 
. With the knowledge of the sequences $\tau$ and $\delta$, we can expand nodes whose $b$-values lower bound $f^*$, discarding other nodes and striking an efficient balance in exploration-vs-exploitation based on the lowest $\tau(h)+\delta(h)$ portion of the function space.

However, the knowledge of such sequences ($\delta,\lambda,\tau$) is not available/known in practice. Thus, we follow an optimistic approach and propose to simulate the knowledge of these sequences via two realization aspects of the algorithm. First, the tree $\mathcal{T}$'s nodes are visited based on their depths and their base points' norms: relating the depth-wise and norm-wise visits to the notion of intra-(same projection) and inter-(different projections) exploration-vs.-exploitation dilemmas, respectively. Second, as the tree $\mathcal{T}$ is swept across multiple depths and norms, a node $(h,i)$ is expanded only if its $f^*_{h,i}$ is strictly smaller than those of nodes of higher depths and those of nodes at the same depth but of greater or equal base points' norms.

Besides motivating the norm-wise traversal described above, Theorem~\ref{lem:gy_def} implies evaluating $f$ at the nodes' base points multiple times---each with a new random projection---in proportion to their norms as larger improvement over the current $f$ value is probable at points with greater norms. A tree $\mathcal{T}$ with an odd-numbered partition factor $K$ can seamlessly accommodate this observation because the center 
child node $(h+1,j)$ of a node $(h,i)$ shares the same base point as its parent $(h,i)$. Therefore, one can decide whether to evaluate a newly created center child node based on the number of function evaluations that its base point had in its ancestor nodes in relation to its norm.

In summary, Algorithm~\ref{alg:eh} describes \ouralg. The algorithm takes four parameters: i). the maximum depth $h_{max}$ up to which nodes can be expanded, it can be a function of the evaluation budget or the number of iterations similar to other optimistic methods; ii). the partition factor $K$, which has to be an odd number; iii). $\eta\in(0,1)$ to specify the bounds of the search space $\mathcal{Y}$ from the random projection theory; and iv). a multiplicative factor $M$ to bound the number of past function evaluations at a base point $\vy_{h,i}$ such that it is not greater than $M\cdot ||\vy_{h,i}||$, simulating Theorem~\ref{lem:gy_def}'s bound, $\sqrt{8}\cdot L \cdot ||\vy_{h,i}||$. Otherwise, no more evaluation is performed and the node retains its parent's best achieved function value~(performed at Line~\ref{ln:eval} of the algorithm). For the sake of readability, the following definitions were used in Algorithm~\ref{alg:eh}.
\begin{eqnarray}
\mathcal{L}_{t,h}&\myeq&\{(h,i)\mid 0\leq i < K^h,\; (h,i) \in \mathcal{L}_t\}\nonumber\\
\Gamma_{h,t}&\myeq&\{\gamma \in \mathbb{R}_0^+\mid \exists (h,i) \in  \mathcal{L}_{t,h}   \text{ such that } \gamma = ||\vy_{h,i}||  \}\nonumber\\
\mathcal{L}^{j}_{t,h}&\myeq&\{(h,i)\mid (h,i) \in \mathcal{L}_{t,h},\;\nonumber\\ &&||\vy_{h,i}||=\text{the $j$th   largest element} \in \Gamma_{h,t} \}
\end{eqnarray}

\newcommand{\Statey}{\State}
\begin{algorithm}[htp]
	\caption{The \ouralg~Algorithm} 	\label{alg:eh}
	\begin{algorithmic}[1]
	\Statex \textbf{Input}: 
		\Statex \hspace{1em} stochastic function $g_P$, 
		\Statex \hspace{1em} search space $\mathcal{Y}=[-d/\eta,d/\eta]^d$, 
		\Statex \hspace{1em} evaluation budget~$v$.
		\Statex \textbf{Initialization}: 
		\Statex \hspace{1em} $t\gets 1$,
	$\mathcal{T}_1=\{(0,0)\}$,
 Evaluate $g_P(\vy_{0, 0})$.
		\While{evaluation budget is not exhausted}
		\Statey  $\nu_{\min} \gets \infty$
		\For{$l=0$ \textbf{~to~}$\min\{\mbox{depth}{(\mathcal{T}_t}), h_{\max}\}$}
		\For{$j=1$ \textbf{~to~}$|\Gamma_{l,t}|$} \label{ln:beg_iteration}
		\Statey Select $(l, o)=\arg\min_{(h,i) \in \mathcal{L}^{j}_{t,l}}f^*_{h,i}$ \label{line:soo_qt}
		\If {$f^*_{l,o} < \nu_{\min}$}
		\Statey $\nu_{\min} \gets f^*_{l,o}$
		\Statey Expand $(l,o)$ into its child nodes
		\Statey Evaluate $(l,o)$'s child nodes by $g_P$\label{ln:eval}
		\Statey Add $(l,o)$'s child nodes to $\mathcal{T}_t$
		\EndIf
		\EndFor 
		\Statey $\mathcal{T}_{t+1}\gets \mathcal{T}_t$
		\Statey $t\gets t + 1$ \label{ln:end_iteration}
		\EndFor
		\EndWhile
		\Statey \textbf{return} $f^*_v=\min_{(h,i)\in \mathcal{T}_t} f^*_{h,i}$
	\end{algorithmic}
	
\end{algorithm}

\section{Theoretical Analysis}
\label{sec:theory}

In this section, we analyze the performance of the \ouralg~algorithm and upper-bound its regret~\eqref{eq:simple-regret}. To derive a bound on the regret, a measure of the quantity of near-optimal points is used, which is closely similar to those in  \cite{bubeck2009online,Aldujaili2016mso} and defined after introducing some terminology.  For any $\epsilon > 0$, define the set of $\epsilon$-optimal points as $\mathcal{Y}_{\epsilon}\myeq \{\vy \in \mathcal{Y}\mid \min_{p}g_p(\vy)\leq f^* + \epsilon  \}$ and let $g(\mathcal{Y}_{\epsilon})\myeq \{\min_{p}g_p(\vy)\mid \vy \in \mathcal{Y}_\epsilon \}=[f^*,f^*+\epsilon]$. Likewise, denote the set of \textit{$\epsilon$-optimal nodes} at depth $h$ whose base points are in $\mathcal{Y}_{\epsilon}$ by 
 $\mathcal{I}^{\epsilon}_h\myeq\{(h,i)\in \mathcal{T}\mid 0\leq i< K^h, \vy_{h,i}\in\mathcal{Y}_\epsilon \}$. 
After $t$ iterations, one can denote the depth of the \textit{deepest expanded optimal} node by $h^*_t$, where one iteration represents \emph{executing the lines~\ref{ln:beg_iteration}--\ref{ln:end_iteration} of Algorithm~\ref{alg:eh}, once}.

		\begin{definition}
			\label{def:near_optimal_dim}
			The \textbf{$m$-near-optimality dimension} is the smallest $d_m \geq 0$ such that there exists $C > 0$ such that for any $\epsilon > 0$, the maximum number of disjoint $\ell$-balls of radius $m\epsilon$ and center in $\mathcal{Y}_{\epsilon}$ is less than $C\epsilon^{-d_m}$.
		\end{definition}


Let the considered depth after $t-1$ iterations be $h$ and the depth of the deepest expanded optimal  node $h^*_{t-1}$ be $h-1$. At iteration $t$, \ouralg~would expand at most $|\Gamma_{h,t}|$ nodes. As the (any) optimal node at depth $h$ is in one of the $\{\mathcal{L}^j_{t,h}\}_{1\leq j \leq |\Gamma_{h,t}|}$ sets, the optimal node at depth $h$ is not expanded at iteration $t$ if $\nu_{\min} \leq f^*_{h,i^*_p}$ or if there exists a node $(h,i)\in \mathcal{L}_{h,t}$ such that $||\vy_{h,i}||\geq||\vy_{h,i^*_p}||$ and $f^*_{h,i} \leq f^*_{h,i^*_p}$. The latter condition implies  $f^*_{h,i} - f^*  \leq  f^*_{h,i^*_p} - f^*$ and by triangular inequality, we have $f^*_{h,i} - f^*  \leq  |f^*_{h,i^*_p} - g_p(\vy_{h,i^*_p})| + |g_p(\vy_{h,i^*_p}) - f^*|$. Hence, from Assumption~\ref{assmptn:lambday} and Assumption~\ref{assmptn:deltah}, $f^*_{h,i} - f^*  \leq  \tau(h) + \delta(h)\nonumber.$
Since $\tau$ and $\delta$ are non-decreasing and decreasing sequences in $h$, respectively. One can write $\tau(h)$ as a multiple of $\delta(h)$, that is, there exists $m_h \in \mathbb{Z}^+$ such that
\begin{align}
f^*_{h,i} - f^*  \leq& \tau(h) + \delta(h) \leq m_h \cdot \delta(h) \label{eq:mh}\;. 
\end{align}
	\begin{figure*}[t!]
		\centering
		\includegraphics[width=0.95\textwidth]{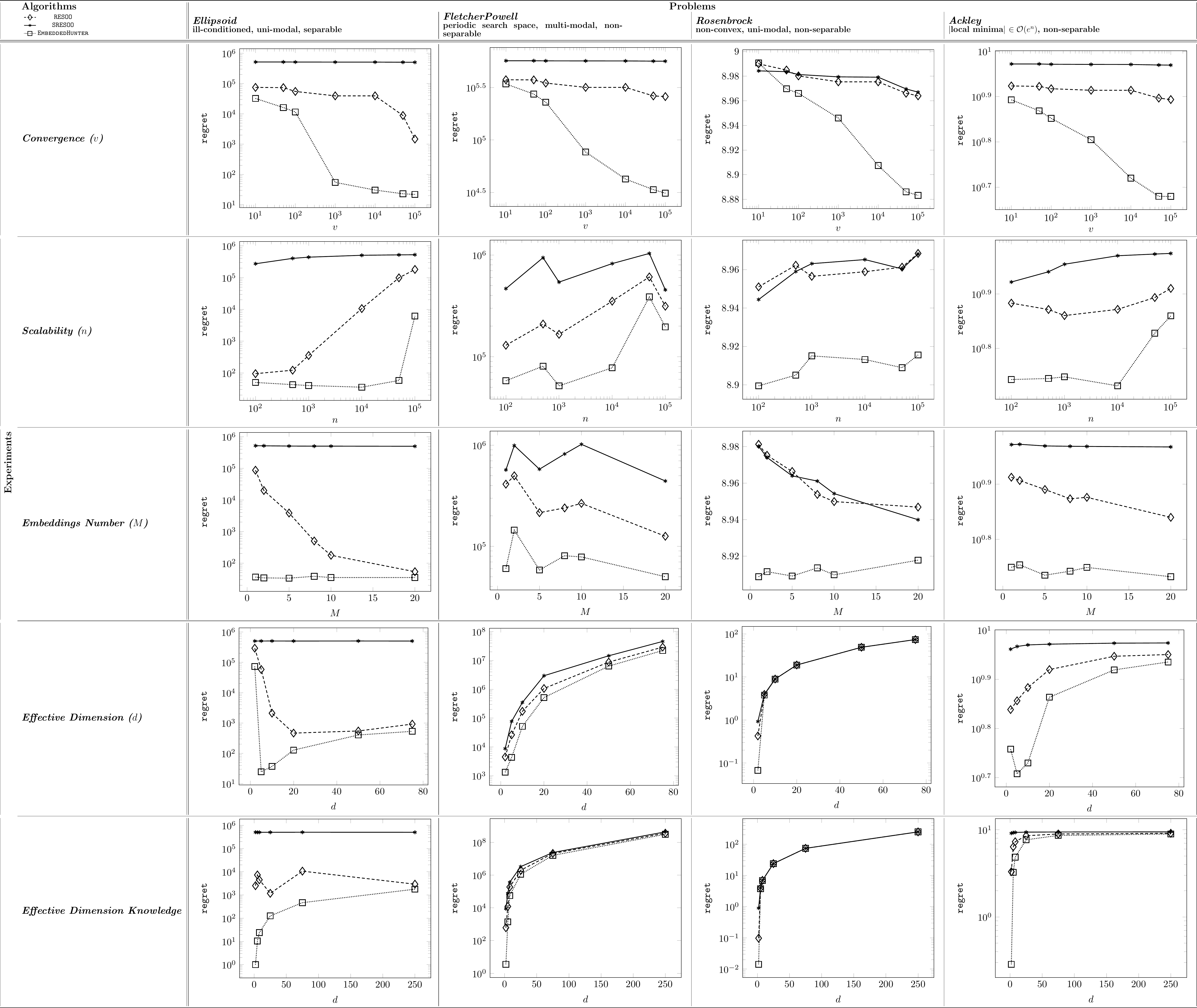}		
		\caption{Empirical validation of the \ouralg~algorithm on a set of commonly-used test optimization problems in comparison with recent large-scale techniques namely \alg{RESOO} and \alg{SRESOO}. Each data point of an algorithm's curve represents the mean performance of its 20 runs w.r.t. the experimental configuration considered.}
		\label{fig:empirical_table}
	\end{figure*}
Put it differently, when $h^*_{t-1}=h-1$, the base point of any node at depth $h$, that is later expanded prior to the optimal node at the same depth, is in the near-optimal space $\mathcal{Y}_{m_h\delta(h)}$. 
Now, if we assume that prior to any iteration at depth $h$, $\nu_{\min}\geq m_{h}\delta(h)$, then by Algorithm~\ref{alg:eh}, it takes at most the next $|\mathcal{I}^{m_h\delta(h)}_{h}|$ iterations at depth $h$ to expand the optimal cell \optcell. With this observation at hand, the next question follows naturally: how many iterations at other depths $\in\{0,\ldots,h_{max}\}$ are required---at most---to expand the optimal cell \optcell, and with no assumption on $\nu_{\min}$? In the following lemma, we show that the number of iterations required is upper bounded by the number of nodes in supersets of each of $\{\mathcal{I}^{m_h\delta(h)}_{h}\}_{0\leq h \leq h_{max}}$. The reader can refer to the supplemental material for a pictorial explanation.


\begin{lemma} Let depth $h\in \{0,h_{max}\}$, $\hat{m} = m_{h_{max}}$ and 
\begin{equation}t_h \myeq h_{max}\big( |\mathcal{I}^{\hat{m}\delta(0)}_{0}| + |\mathcal{I}^{\hat{m}\delta(1)}_{1}| + \cdots + |\mathcal{I}^{\hat{m}\delta(h)}_{h}|\big)\;. \label{eq:def_t_h}
\end{equation}
After $t\geq t_h$ iterations, the depth of the deepest expanded optimal node is at least $h$, \ie, $h^*_t\geq h$. 
\label{lem:iter_bound}
\end{lemma}
\begin{proof}
		Refer to the supplemental material.
\end{proof}

Lemma~\ref{lem:iter_bound} quantifies the number of iterations required to expand an optimal node as a function of the number of $\hat{m}\delta(h)$-optimal nodes. Based on Assumption~\ref{assmptn:rounded}, the following lemma upper bounds the cardinality of such nodes.
	
\begin{lemma} Let depth $h\in \{0,h_{max}\}$, we have  $|\mathcal{I}^{\hat{m}\delta(h)}_h| \leq C (\hat{m}\delta(h))^{-\hat{d}}$, where $\hat{d}$ is defined as the $m/\hat{m}$-near-optimality dimension and $C$ the related constant.
	\label{lem:I_bound}
\end{lemma}
\begin{proof}
		Refer to the supplemental material.
\end{proof}
With Lemmas~\ref{lem:iter_bound} and~\ref{lem:I_bound} at hand, the finite-time regret~\eqref{eq:simple-regret}~of the~\ouralg~algorithm can be linked to the number of iterations as presented in the next theorem.
\begin{theorem}
Define $h(t)$ as the smallest $h\geq0$ such that:	
\begin{equation}
Ch_{max}\sum\limits_{l=0}^{h(t)}(\hat{m}\delta(l))^{-\hat{d}} \geq t\;, \label{eq:thm_eh} 
\end{equation}
where $t$ is the number of iterations.
Then \ouralg's regret~is bounded as
$r(t)\leq \min\{\tau(h)+\delta(h)\mid h\leq \min(h(t),h_{max}+1)\}\;.\label{eq:loss_bound}$
\label{thm:eh}	
\end{theorem}
\begin{proof}
	Refer to the supplemental material.
\end{proof}

\section{Empirical Analysis}
\label{sec:empirical}
In this section, the efficacy of the proposed method is empirically validated on a set of scalable functions from the literature: 
the Ellipsoid, FletcherPowell, Rosenbrock, and Ackley test functions~\cite{eliipsoid}, each of which reflects some challenges in black-box optimization, \eg, modality, separability, and conditioning. The proposed optimistic method is also compared with the scaled \alg{SOO} optimistic algorithm~\cite{munos-soo} within two recently presented methods in the random-embedding multiple-run framework, viz. the Simultaneous Optimistic Optimization with Random Embedding~(\alg{RESOO})~\cite{qian2016scaling}  and  Simultaneous Optimistic Optimization with Sequential Random Embedding~(\alg{SRESOO})~\cite{qianderivative} algorithms. With the aim of fully characterizing the algorithms' performance, five experiments are conducted with respect to (w.r.t) the performance as listed in Table~\ref{tbl:exp-setup}.

\vspace{0.5em}
\noindent\textbf{Experiment Setup.} The compared algorithms were implemented in \textit{Python} and the  test functions were imported from the \textit{Optproblems Python package}~\cite{optproblems-package}. Each algorithm is run 20 times
independently per an experiment configuration and the average performance is reported. The experiments were set up as listed in Table~\ref{tbl:exp-setup}. The code/data/supplemental materials of this paper will be made available at the project's website:

\noindent\url{http://ash-aldujaili.github.io/eh-lsopt}.
\begin{table}[]
	\resizebox{0.47\textwidth}{!}{
	\begin{tabular}{p{6.75cm}p{6.5cm}}
		\toprule
		\textbf{Experiment} & \textbf{Setup} \\\toprule
		Convergence: performance w.r.t the number of function evaluations $v$ & $v\in \{10,50,10^2,10^3,10^4,5\times10^4,10^5 \}$\\\midrule
		Scalability: performance w.r.t the problem's dimensionality $n$ & $n\in \{10^2, 5\times 10^2, 10^3, 10^4, 5\times 10^4, 10^5 \}$ \\\midrule
		Embedding Number: performance w.r.t. the number of times a random matrix is sampled $M$ & $M\in \{1, 2, 5, 8, 10, 20 \}$ \\\midrule
		Effective Dimension: performance w.r.t. the problem's implicit dimensionality $d$ & $d\in \{2,5,10,20,50,75 \}$ \\\midrule
		Effective Dimension Knowledge: performance w.r.t the mismatch between  the low dimension used and the actual effective dimension & $d=\{2,5,8,25,75,250\}$, $\mathcal{Y}=[-d/\eta,d/\eta]^{10}$ for \alg{RESOO} and \ouralg~and $[-1,1]^{10}$ for \alg{SRESOO}. \\	\midrule
	\end{tabular}
} \caption{Experiments setup. Unless specified above, we set $v=10^4$, $n=10^4$, $d=10$, and $M=5$. The search space $\mathcal{Y}$ for \alg{RESOO} and \ouralg~was set to $[-d/\eta,d/\eta]^d$ with $\eta=0.3$, \alg{SRESOO}'s $\mathcal{Y}$ was set to $[-1,1]^{d+1}$ as suggested in \cite{qian2016scaling,qianderivative}, respectively. $h_{max}$ was set to the square root of number of function evaluations for each tree. \ie, for \alg{RESOO} and \alg{SRESOO}, $h_{max}=\sqrt{v/M}$; for~\ouralg, $h_{max}=\sqrt{v}$. }\label{tbl:exp-setup}
\end{table}

\vspace{0.5em}
\noindent\textbf{Results \& Discussion.} Results from the five experiments on the four test functions are presented in Figure~\ref{fig:empirical_table}.  
One can easily appreciate \ouralg's performance w.r.t the compared algorithms. 

\textit{Convergence ($v$)}. Across all the tested functions, 
 the performance gap between the algorithms grows larger with higher evaluation budget $v$. With more function evaluations, \ouralg~is able to further refine its best achieved solution in comparison with the other algorithms.

\textit{Scalability ($n$)}. As expected, the best solution quality degrades with the problem's explicit dimensionality $n$. Nevertheless, the performance of \alg{SRESOO} looks robust yet poorer than that of \ouralg~and \alg{RESOO}.

 \textit{Embedding Number ($M$)}. Allocating more independent runs seems to be effective for \alg{RESOO}'s and of lesser effect to \alg{SRESOO}'s performance. As $M$ approaches  $v$, both the algorithms act as random search optimizers. This is not the case for \ouralg, which uses $M$ as a multiplicative factor to bound the number of evaluations a base point $\vy_{h,i}$ can have up to $M\cdot ||\vy_{h,i}||$.~\ouralg~uses $M$ as an estimate of Theorem~\ref{lem:gy_def}'s bound factor $\sqrt{8}\cdot L$. Thus, there's a sweet-spot value for each function, which explains the regret's variations for each function in $M$.

 \textit{Effective Dimension ($d$)}. The results validate the assumption of low effective dimensionality for the suitability of random embedding technique for large-scale problems. It does not scale well with higher effective dimensionality and the algorithms' performance gap reduces w.r.t. the same.

\textit{Knowledge of the Effective Dimension}. This experiment investigated the performance of low-dimensional embedding irrespective of the effective dimension---be it higher or lower (see Table~\ref{tbl:exp-setup}).   
As the mismatch between the two quantities increases, 
the performance degrades and the gap among the algorithms reduces.

\section{Conclusion}
\label{sec:conclusion}
This paper has presented the \ouralg~algorithm, a different approach to random embedding for large-scale black-box optimization. While the bulk of random embedding techniques in the literature employ the multiple-run paradigm sampling a new random projection for each run to maximize the probabilistic guarantee of convergence to the optimal solution, \ouralg~looks for the optimal solution by building stochastic hierarchical bandits (so-called a tree) over a low-dimensional search space $\mathcal{Y}$, where stochasticity has shown to be proportional on average with the norm of the nodes' base points. 

The distinctive advantage of \ouralg~is that its search tree implicitly ranks $\mathcal{Y}$'s regions via its depth/norm-wise visits and allocates the evaluation budget accordingly. Indeed, other algorithms (e.g., \alg{RESOO}) may evaluate $\mathcal{Y}$'s center--which is a zero vector--$M$ times in its $M$ independent tree searches/projections. This is inefficient as $M$ evaluations are spent generating the same function value, whereas~\ouralg~evaluates the zero vector once and reserves the rest ($M-1$) evaluations to points with greater norms, exploring more values in the function space.

The finite-time analysis of the algorithm has characterized its performance in terms of the regret as a function of the number of iterations. Besides its theoretically-proven performance, the numerical experiments have validated \ouralg's efficacy and robustness with regard to recent random-embedding methods.

\section{ Acknowledgments}
The research was partially supported by the ST Engineering - NTU Corporate Lab, Singapore, through the NRF corporate lab@universty scheme.

\bibliographystyle{aaai}
\bibliography{refs}

\begin{thebibliography}{}

\bibitem[\protect\citeauthoryear{Achlioptas}{2003}]{Achlioptas2003671}
Achlioptas, D.
\newblock 2003.
\newblock Database-friendly random projections: Johnson-lindenstrauss with
  binary coins.
\newblock {\em Journal of Computer and System Sciences} 66(4):671 -- 687.
\newblock Special Issue on {PODS} 2001.

\bibitem[\protect\citeauthoryear{Al-Dujaili and
  Suresh}{2016}]{AlDujaili2016nmso}
Al-Dujaili, A., and Suresh, S.
\newblock 2016.
\newblock A naive multi-scale search algorithm for global optimization
  problems.
\newblock {\em Information Sciences} 372:294 -- 312.

\bibitem[\protect\citeauthoryear{Al-Dujaili, Suresh, and
  Sundararajan}{2016}]{Aldujaili2016mso}
Al-Dujaili, A.; Suresh, S.; and Sundararajan, N.
\newblock 2016.
\newblock {MSO}: a framework for bound-constrained black-box global
  optimization algorithms.
\newblock {\em Journal of Global Optimization}  1--35.

\bibitem[\protect\citeauthoryear{Bergstra and
  Bengio}{2012}]{bergstra2012random}
Bergstra, J., and Bengio, Y.
\newblock 2012.
\newblock Random search for hyper-parameter optimization.
\newblock {\em Journal of Machine Learning Research} 13(Feb):281--305.

\bibitem[\protect\citeauthoryear{Bubeck \bgroup et al\mbox.\egroup
  }{2009}]{bubeck2009online}
Bubeck, S.; Stoltz, G.; Szepesv{\'a}ri, C.; and Munos, R.
\newblock 2009.
\newblock Online optimization in $\mathcal{X}$-armed bandits.
\newblock In {\em Advances in Neural Information Processing Systems},
  201--208.

\bibitem[\protect\citeauthoryear{Carpentier, Munos, and
  others}{2012}]{carpentier2012bandit}
Carpentier, A.; Munos, R.; et~al.
\newblock 2012.
\newblock Bandit theory meets compressed sensing for high dimensional
  stochastic linear bandit.
\newblock In {\em AISTATS},  190--198.

\bibitem[\protect\citeauthoryear{Chen \bgroup et al\mbox.\egroup
  }{2010}]{chen2010large}
Chen, W.; Weise, T.; Yang, Z.; and Tang, K.
\newblock 2010.
\newblock Large-scale global optimization using cooperative coevolution with
  variable interaction learning.
\newblock In {\em International Conference on Parallel Problem Solving from
  Nature},  300--309.
\newblock Springer.

\bibitem[\protect\citeauthoryear{Chen, Krause, and
  Castro}{2012}]{ICML2012Chen_714}
Chen, B.; Krause, A.; and Castro, R.~M.
\newblock 2012.
\newblock Joint optimization and variable selection of high-dimensional
  gaussian processes.
\newblock In Langford, J., and Pineau, J., eds., {\em Proceedings of the 29th
  International Conference on Machine Learning (ICML-12)},  1423--1430.
\newblock New York, NY, USA: ACM.

\bibitem[\protect\citeauthoryear{Djolonga, Krause, and
  Cevher}{2013}]{Djolonga-high-bandits}
Djolonga, J.; Krause, A.; and Cevher, V.
\newblock 2013.
\newblock High-dimensional gaussian process bandits.
\newblock In Burges, C.; Bottou, L.; Welling, M.; Ghahramani, Z.; and
  Weinberger, K., eds., {\em Advances in Neural Information Processing Systems
  26}.
\newblock  1025--1033.

\bibitem[\protect\citeauthoryear{Friesen and
  Domingos}{2015}]{friesen2015recursive}
Friesen, A.~L., and Domingos, P.
\newblock 2015.
\newblock Recursive decomposition for nonconvex optimization.
\newblock In {\em Proceedings of the 24th International Joint Conference on
  Artificial Intelligence},  253--259.

\bibitem[\protect\citeauthoryear{Hansen}{1988}]{hansen19882}
Hansen, P.~C.
\newblock 1988.
\newblock The 2-norm of random matrices.
\newblock {\em Journal of computational and applied mathematics}
  23(1):117--120.

\bibitem[\protect\citeauthoryear{Kaban, Bootkrajang, and Durrant}{2013}]{kaban}
Kaban, A.; Bootkrajang, J.; and Durrant, R.~J.
\newblock 2013.
\newblock Towards large scale continuous eda: A random matrix theory
  perspective.
\newblock In {\em Proceedings of the 15th Annual Conference on Genetic and
  Evolutionary Computation}, GECCO '13,  383--390.
\newblock New York, NY, USA: ACM.

\bibitem[\protect\citeauthoryear{Kandasamy, Schneider, and
  P{\'o}czos}{2015}]{kandasamy2015high}
Kandasamy, K.; Schneider, J.; and P{\'o}czos, B.
\newblock 2015.
\newblock High dimensional bayesian optimisation and bandits via additive
  models.
\newblock In {\em International Conference on Machine Learning (ICML)}.

\bibitem[\protect\citeauthoryear{Kocsis and
  Szepesv{\'a}ri}{2006}]{kocsis2006bandit}
Kocsis, L., and Szepesv{\'a}ri, C.
\newblock 2006.
\newblock Bandit based monte-carlo planning.
\newblock In {\em Machine Learning: ECML 2006}. Springer.
\newblock  282--293.

\bibitem[\protect\citeauthoryear{Molga and Smutnicki}{2005}]{eliipsoid}
Molga, M., and Smutnicki, C.
\newblock 2005.
\newblock Test functions for optimization needs.
\newblock http://www.zsd.ict.pwr.wroc.pl/files/docs/functions.pdf.
\newblock [Retrieved June 2013].

\bibitem[\protect\citeauthoryear{Munos}{2011}]{munos-soo}
Munos, R.
\newblock 2011.
\newblock Optimistic optimization of a deterministic function without the
  knowledge of its smoothness.
\newblock In {\em Advances in Neural Information Processing Systems 24},
  783--791.
\newblock Curran Associates, Inc.

\bibitem[\protect\citeauthoryear{Qian and Yu}{2016}]{qian2016scaling}
Qian, H., and Yu, Y.
\newblock 2016.
\newblock Scaling simultaneous optimistic optimization for high-dimensional
  non-convex functions with low effective dimensions.
\newblock In {\em Proceedings of the 30th AAAI Conference on Artificial
  Intelligence ({AAAI'16})}.

\bibitem[\protect\citeauthoryear{Qian, Hu, and Yu}{2016}]{qianderivative}
Qian, H.; Hu, Y.-Q.; and Yu, Y.
\newblock 2016.
\newblock Derivative-free optimization of high-dimensional non-convex functions
  by sequential random embeddings.
\newblock In {\em Proceedings of the 25th International Joint Conference on
  Artificial Intelligence ({IJCAI'16})}.

\bibitem[\protect\citeauthoryear{Vempala}{2004}]{vempala2005random}
Vempala, S.~S.
\newblock 2004.
\newblock {\em The Random Projection Method}, volume~65.
\newblock Providence, Rhode Island: American Mathematical Society.

\bibitem[\protect\citeauthoryear{Wang \bgroup et al\mbox.\egroup
  }{2013}]{wang2013bayesian}
Wang, Z.; Zoghi, M.; Hutter, F.; Matheson, D.; Freitas, N.; et~al.
\newblock 2013.
\newblock Bayesian optimization in high dimensions via random embeddings.
\newblock In {\em Proceedings of the 23rd International Joint Conference on
  Artificial Intelligence ({IJCAI'13})},  1778--1784.
\newblock AAAI Press/International Joint Conferences on Artificial
  Intelligence.

\bibitem[\protect\citeauthoryear{Wessing}{2016}]{optproblems-package}
Wessing, S.
\newblock 2016.
\newblock {Optproblems}: Infrastructure to define optimization problems and
  some test problems for black-box optimization.
\newblock [Online; accessed 2016-09-07].

\end{thebibliography}


\begin{thebibliography}{1}

\bibitem{GablonskyModirect}
J.M. Gablonsky.
\newblock {\em Modifications of the Direct Algorithm}.
\newblock PhD thesis, North Carolina State University, Raleigh, North Carolina,
  2001.

\bibitem{Pinter1995}
J{\'a}nos Pint{\'e}r.
\newblock {\em Global optimization in action: continuous and Lipschitz
  optimization: algorithms, implementations and applications}, volume~6.
\newblock Springer Science \& Business Media, 1995.

\end{thebibliography}

\end{document}


\maketitle



This document is a supplement to the paper titled "\emph{Embedded Bandits for Large-Scale Black-Box Optimization}", which we will refer to as the main paper subsequently. The objective of this material is to present some remarks, visual description and proofs to the discussion and results in the main paper.
\appendix

\renewcommand{\thesubsection}{A.\arabic{subsection}}

\subsection{On the function values difference within same/multiple projections} 
The “same projection/among different projections” concepts are incorporated in the algorithm via its depth/norm-wise visits and the limit on the number of function evaluations per base point. In Page 4 of the main paper, these concepts are related to the algorithm’s aspects. They were presented separately in Page 3 of the main paper to indicate the existence of two exploration-exploitation trade-offs: (in one, among multiple) projection(s). 

\subsection{How do \alg{RESOO}, \alg{SRESOO}, and \ouralg~use their evaluation budget $v$?}
Given $v$ function evaluations,1). \alg{RESOO}: $M$ random matrices are used in $M$ independent SOO-searches, each with $v/M$ evaluations; 2). \alg{SRESOO}: $M$ random matrices are used in $M$ sequential SOO-searches, each with $v/M$ evaluations, where the ($s+1$)th search optimizes an augmented objective function based on the best solution of the $s$th search; 3). Instead of multiple tree searches, \ouralg~consolidates the search in a single tree with $v$ evaluations but allows to evaluate a single base point $\vy_{h,i}$ multiple times (each with a new sample of the projection matrix). The number of evaluations a point can have is proportional to its norm (Theorem 1). To relate to the embedding number $M$, a base point can be evaluated as long as the number of its past evaluations is not greater than $M\cdot ||\vy_{h,i}||$ simulating  Theorem 1's bound, $\sqrt{8}\cdot L\cdot||\vy_{h,i}||$. 

\subsection{On the Generality of Assumption 1}

	The class of functions that satisfies the Lipschitz condition is very broad. In fact, it has been shown in \cite{Pinter1995,GablonskyModirect} that among the Lipschitz-continuous functions are convex/concave functions over a closed domain and continuously differentiable functions.

\subsection{Visual Description and Proof of Lemma 1}

\begin{figure}[]
	\centering
	\begin{tabular}{c}
	\includegraphics[width=0.95\textwidth]{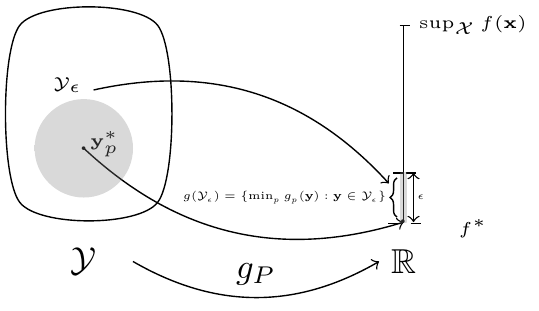}\\ (a) \\ \includegraphics[width=0.95\textwidth]{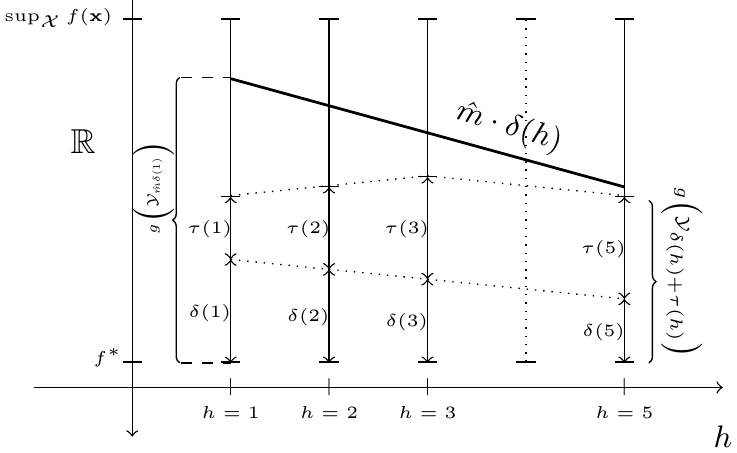}	\\ (b)\\
	\end{tabular}		
	\caption{(a) The volume of the near-optimal search space at depth $h$ in \ouralg~where $\epsilon=\tau(h)+\delta(h)$. A mapping from the low-dimensional search space $\mathcal{Y}$ to the space of objective function $f$ can be obtained via a realization $g_p$ of the random function $g_P$. (b) The range of function values for the $\tau(h)+\delta(h)$-optimal nodes $\mathcal{I}^{\tau(h)+\delta(h)}_h$ for $h=\{1,2,3,5\}$. One can upper bound the sequence $\{\tau(h)+\delta(h)\}_{0\leq h\leq h_{max}}$ by the decreasing sequence $\{\hat{m}\delta(h)\}_{0\leq h \leq h_{max}}$, where $\hat{m}=m_{h_{max}}$.}\label{fig:tau-delta}
\end{figure}

Consider Figure~\ref{fig:tau-delta}, where all nodes at depth~$h=0$~have been expanded with out loss of generality. One can see that if all the nodes in $\mathcal{I}^{m_1\delta(1)}_1$ have been expanded, then in later iterations at depth $h=1$ nodes from $\mathcal{L}_{t,1}\setminus \mathcal{I}^{m_1\delta(1)}_1$ will be expanded. As a result, prior to iterations at depth $h=2$, the minimum value, $\nu_{\min}$ can have, is $f^*+m_1\delta(1)$. Since $\tau$ is non-decreasing in $h$, $m_1\leq m_2$ and it is possible that there exists some node $(h,i)$ in $\mathcal{I}^{m_2\delta(2)}_{2}$ such that $\nu_{\min} \leq f^*_{h,i} \leq f^*+m_2\delta(2)$, and hence it will not be expanded. In other words, we are certain that all the nodes in $\mathcal{I}^{m_2\delta(2)}_{2}$ will be expanded if $\nu_{\min}$---prior to any iteration at depth $h=2$---is greater than  $f^*+m_2\delta(2)$.
In the light of this observation, the following lemma is deduced.
\begin{lemma} Let depth $h\in \{0,h_{max}\}$, $\hat{m} = m_{h_{max}}$ and 
	\begin{equation}t_h \myeq h_{max}\big( |\mathcal{I}^{\hat{m}\delta(0)}_{0}| + |\mathcal{I}^{\hat{m}\delta(1)}_{1}| + \cdots + |\mathcal{I}^{\hat{m}\delta(h)}_{h}|\big)\;. \label{eq:def_t_h}
	\end{equation}
	After $t\geq t_h$ iterations, the depth of the deepest expanded optimal node is at least $h$, \ie, $h^*_t\geq h$. 
	\label{lem:iter_bound}
\end{lemma}
\begin{proof}
	First, the lemma holds trivially for $h=0$ as $h^*_t\geq 0$. For $h>0$, the proof is presented by induction. To this end, let the lemma holds for all $h\leq \hat{h}<h_{max}$, and we need to show it holds for $\hat{h}+1$. Assume that $t_{\hat{h}+1}$ iterations have been performed, that is to say, the present iteration is $t\geq t_{\hat{h}+1}$. As $t\geq t_{\hat{h}+1} \geq t_{\hat{h}}$, the induction assumption implies that $h^*_t \geq \hat{h}$. Furthermore, the induction assumption implies that $\nu_{\min}> \hat{m}\delta(\hat{h})$ prior to any iteration at depth $\hat{h}+1$ because all the nodes in $\{\mathcal{I}^{\hat{m}\delta(h)}_{h}\}_{0\leq h \leq \hat{h}}$ have been expanded in previous iterations. From Eq.~(7) of the main paper 
	  and the definition of $\mathcal{I}^{m_{\hat{h}+1}\delta(\hat{h}+1)}_{\hat{h}+1}$, we are certain that $h^*_t\geq \hat{h}+1$ if all the nodes in $\mathcal{I}^{m_{\hat{h}+1}\delta(\hat{h}+1)}_{\hat{h}}$ have expanded. To guarantee their expansions, we need an additional total number of iterations across all depths greater than or equal to $|\mathcal{I}^{\hat{m}\delta(\hat{h}+1)}_{\hat{h}+1}|\geq |\mathcal{I}^{m_{\hat{h}+1}\delta(\hat{h}+1)}_{\hat{h}+1}|$ (see Figure~\ref{fig:tau-delta}) times the tree depth  $h_{max}$ (one iteration per depth). In total, the number of iterations is equal to $t_{\hat{h}+1}$ and thus $h^*_t\geq \hat{h}+1$. 
\end{proof}

\subsection{Proof of Lemma 2}
\begin{lemma} Let depth $h\in \{0,h_{max}\}$, we have  $|\mathcal{I}^{\hat{m}\delta(h)}_h| \leq C (\hat{m}\delta(h))^{-\hat{d}}$, where $\hat{d}$ is defined as the $m/\hat{m}$-near-optimality dimension and $C$ the related constant.
	\label{lem:I_bound}
\end{lemma}

\begin{proof}
	The proof is made by contradiction. To this end, assume there exists some $h \in \{0,h_{max}\}$ such that $|\mathcal{I}^{\hat{m}\delta(h)}_h| > C (\hat{m}\delta(h))^{-\hat{d}}$. On the one hand,
	the definition of $\mathcal{I}^{\hat{m}\delta(h)}_h$ indicates that their base points are in $\mathcal{Y}_{\hat{m}\delta(h)}$. On the other hand, Assumption~3 of the main paper 
	indicates that cells of nodes at depth $h$ contain a ball of radius $m\delta(h)=\frac{m}{\hat{m}}\cdot \hat{m} \delta(h)$. Since the cells are disjoint and from Definition 2 of the main paper,
	 we have a contradiction with $\hat{d}$ being the $m/\hat{m}$-near-optimality dimension.
\end{proof}

\subsection{Proof of Theorem 2}
\setcounter{theorem}{1}
\begin{theorem}($r(t)$ for \ouralg) 
	Let us define $h(t)$ as the smallest $h\geq0$ such that:	
	\begin{equation}
	Ch_{max}\sum\limits_{l=0}^{h(t)}(\hat{m}\delta(l))^{-\hat{d}} \geq t \label{eq:thm_eh}
	\end{equation}
	where $t$ is the number of iterations.
	Then the regret of \ouralg~is bounded as:
	\begin{equation}
	r(t)\leq \min\{\tau(h)+\delta(h)\mid h\leq \min(h(t),h_{max}+1)\}\;.\label{eq:loss_bound}
	\end{equation}
	\label{thm:eh}	
\end{theorem}

\begin{proof}
	From the definition of $h(t)$ and Lemma~\ref{lem:I_bound}, a bound on $t_{h(t)-1}$ of Eq.~\eqref{eq:def_t_h} can be written as follows.
	\begin{eqnarray}
	t_{h(t)-1} &= &h_{max} \sum_{l=0}^{h(t)-1}|\mathcal{I}^{\hat{m}\delta(l)}_{l}| \nonumber\\
	&\leq &Ch_{max}\sum\limits_{l=0}^{h(t)-1}(\hat{m}\delta(l))^{-\hat{d}} \nonumber\\
	&<&t\;.\nonumber
	\end{eqnarray}
	Then, by Lemma~\ref{lem:iter_bound} and the fact that \ouralg~does not expand nodes beyond $h_{max}$,  $h^*_t\geq \min(h(t)-1,h_{max})$.  Thus, we know that the optimal branch of nodes $\{\mathcal{Y}_{h,i^*_p}\}_{0\leq h \leq \min(h(t),h_{max}+1)}$ have been visited and evaluated at least once and for which the best function value achieved among $\{f^*_{h,i^*_p}\}_{0\leq h \leq \min(h(t),h_{max}+1)}$ is at most $\min\{\tau(h)+\delta(h)\mid h\leq \min(h(t),h_{max}+1)\}$ away from the optimal value $f^*$. Therefore, the regret of the algorithm is upper bounded as in Eq.~\ref{eq:loss_bound}.\end{proof}
\bibliographystyle{plain}
\bibliography{refs}